\documentclass[letterpaper]{article}
\usepackage{aaai}
\usepackage{times}
\usepackage{helvet}
\usepackage{courier}
\usepackage{latexsym}
\usepackage{verbatim}
\usepackage{multirow}
\usepackage{graphicx}
\usepackage{hyperref}
\usepackage{tikz}
\usepackage{color}
\usepackage{amssymb,amsmath,amsthm}
\usepackage{comment}
\usepackage{multirow}
\usepackage{bm}
\usepackage{comment}
\usepackage{multirow}
\usepackage[T1]{fontenc}
\usepackage{caption}
\usepackage{subcaption}

\usepackage{algorithm,algorithmic}

\newcommand{\newcite}[1]{\citeauthor{#1} \shortcite{#1}}

\newcommand{\kmeans}{\textsc{kmeans}}
\newcommand{\kmeanspp}{\textsc{kmeans++}}
\newcommand{\kmeansrand}{\textsc{kmeansRand}}

\newcommand{\kmeansdpp}{\textsc{kmeansD++}}
\newcommand{\kkmeansdpp}{\textsc{kmeansD$_k$++}}
\newcommand{\deter}[0]{determinantal}

\newtheorem{lemma}{Lemma}

\begin{document}
%
\title{Notes on Using Determinantal Point Processes for Clustering with Applications to Text Clustering}

\author{Apoorv Agarwal\\Columbia University\\New York, NY, USA\\ \texttt{apoorv@cs.columbia.edu } \And Anna Choromanska\\Courant Institute of Mathematical Sciences\\New York, NY, USA\\ \texttt{achoroma@cims.nyu.edu} \And Krzysztof Choromanski\\Google Research\\New York, NY, USA\\ \texttt{kchoro@gmail.com}}

\maketitle
\begin{abstract}
\begin{quote}
In this paper, we 
compare three initialization schemes for the
{\kmeans} clustering algorithm: 1) random initialization ({\kmeansrand}),
2) {\kmeanspp}, and 3) {\kmeansdpp}.
Both {\kmeansrand} and {\kmeanspp} have a major 
that the value of $k$ needs to be set by the user
of the algorithms.~\cite{NIPS2013_5008} recently proposed a novel
use of determinantal point processes
for sampling the initial centroids for the {\kmeans} 
algorithm (we call it {\kmeansdpp}). 
They, however, do not provide any
evaluation establishing that {\kmeansdpp} is better
than other algorithms. In this paper, we show
that the performance of {\kmeansdpp} is comparable  to {\kmeanspp}
(both of which are better than {\kmeansrand}) with {\kmeansdpp}
having an additional that it can automatically approximate the
value of $k$.
\end{quote}
\end{abstract}

\section{Introduction}
\textsc{Clustering} is one of the most challenging problems in machine learning due to the lack of supervision and difficulty to evaluate its quality. Its aim is to partition the data into groups, called clusters, such that the members of each group are more similar to each other than to the members of any other group under some measure of similarity, e.g. Euclidean distance. Among many clustering algorithms, {\kmeans} algorithm~\cite{Lloyd:2006:LSQ:2263356.2269955}, also known as Lloyd's algorithm, is one of the most widely-used, simple and easy to
implement clustering algorithm that  works well in practice. 
However, it has no theoretical guarantees in terms
of how far the resulting clustering is from the optimal
clustering. \newcite{Arthur:2007:KAC:1283383.1283494}
proposed an algorithm called {\bf {\kmeanspp}}
that \emph{samples} the initial centroids for the clustering
algorithm from among the data points in a way that
the {\kmeans} clustering algorithm is able to 
achieve theoretical guarantees. The underlying idea
is that sampling takes into account 
the Euclidean distance between points --
higher
the distance between a candidate data point from the already selected centroids, 
higher the probability of selecting this data point as an initial centroid. 
However, there is one major limitation: 
the number of clusters $k$
needs to be determined by the user of the algorithm. 

In this paper, we consider an alternative sampling scheme to the {\kmeanspp} algorithm, a new technique of sampling
the initial set of centroids 
for the {\kmeans} clustering algorithm
that overcomes the aforementioned limitation. The new approach was proposed only recently~\cite{NIPS2013_5008} and uses determinantal point processes (DPPs)~\cite{Kulesza2012} 
for sampling. However, the main
 focus of their paper was to speed up the DPP sampling algorithm. 
\newcite{Reichart2013} use 
DPPs to cluster verbs with similar sub-categorization
frames and selectional preferences.
However, their presentation of the clustering
technique is tied to the task and 
not presented as a general clustering strategy.
Neither of the aforementioned works compare the 
DPP initializer with the {\kmeanspp} initializer
and hence do not provide evidence 
that one has advantages over the other. 

The DPP sampling procedure has a desirable property
(for initializing the {\kmeans} algorithm)
that it samples a \emph{diverse} sub-set of points~\cite{Kulesza2012}.
 In spirit, the notion of diversity in the context
 of DPPs  
 is similar to the notion of Euclidean distance 
 in the context of {\kmeanspp} (DPP samples diverse points while {\kmeanspp} samples points that are far in terms
 of Euclidean distance).
We explore this conceptual connection
between the two sampling techniques 
and provide empirical evidence that {\kmeansdpp}
is as good as {\kmeanspp} with additional advantages.
In the settings, where {\kmeanspp} cannot be used, we compare {\kmeansdpp} algorithm with the randomly initialized {\kmeans} algorithm, which we call {\kmeansrand}, and show superior performance of the former.  We show results on a synthetic data-set and
 a text clustering task that 
 was the motivation for
 us to develop a technique to approximate $k$ 
 for a data-set automatically.




\section{Related work}
\label{sec:related}


In this paper we primarily focus on the center-based clustering problem where the large dataset can be fairly well represented by a small set of cluster centers, e.g. a cluster center can be a convex combination of the data points in this cluster (we will denote the number of clusters as $k$) or the most 'representative' data point from among cluster data points. The most popular clustering algorithm that is used in this setting is the {\kmeans} algorithm and its soft version, Expectation-Maximization (EM)~\cite{Dempster77maximumlikelihood,Liang09onlineem}. Despite their simplicty, both algorithms suffer many problems which prevents their usage in practical problems. They have no theoretical performance guarantees and the solution they recover is extremely sensitive to initialization which usually is done uniformly at random (the solution they converge to can be arbitrarily bad)~\cite{vonLuxburg:2010,Arthur:2007:KAC:1283383.1283494}. Also, they may lead to potential instability~\cite{Bubeck:2009,conf/colt/ShamirT08,conf/nips/RakhlinC06,10.1109/TPAMI.2006.226}.

There only exists few successful attempts to improve the performance of the {\kmeans} algorithm in such a way that the resulting method does have theoretical performance guarantees, meaning it provably approximates a certain measure of clustering quality such as an objective function\footnote{Standard theoretical guarantees show that the objective function to which the algorithm converges is upper-bounded by the optimal value of the objective function multiplied by some bounded small constant greater than $1$.}.

The most widely-cited objective function used to measure the quality of a center-based clustering is the $k$-means clustering objective which is computed as the sum of the squared distances between every data point and its closest cluster center. Optimizing this objective is an NP-hard problem~\cite{journals/ml/AloiseDHP09} and there only exists a few algorithms that provably approximate it \cite{Arthur:2007:KAC:1283383.1283494,Ailon}.
The most popular among them is the {\kmeanspp} algorithm which achieves approximation factor $\mathcal{O}(\log k)$. Other algorithms, this time with constant approximation with respect to the same objective, that were published in the literature include (i) the \textsc{kmeans$\#$} algorithm~\cite{Ailon} which, as opposed to the {\kmeanspp} algorithm, returns more than $k$ centers ($\mathcal{O}(k\log k)$), (ii) adaptive sampling-based approach~\cite{AggarwalDK09}, which returns $\mathcal{O}(k)$ centers, (iii) local search technique~\cite{conf/compgeom/KanungoMNPSW02}, and (iv) online clustering with experts algorithm~\cite{journals/jmlr/ChoromanskaM12}. 

Other notable clustering approaches mainly focus on minimizing other, often less-descriptive to the center-based clustering problem, objectives (e.g. $k$-center or $k$-medoid objective)~\cite{Beygelzimer06covertrees,Charikar:1997:ICD:258533.258657,Guha:2003:CDS:776752.776777}. Among these techniques also spectral methods~\cite{Luxburg:2007:TSC:1288822.1288832,conf/alt/ChoromanskaJKMM13} are widely-cited however they have a much more general scope than the center-based clustering problem and therefore will not be discussed in this paper.

\section{Determinantal Point Processes (DPPs)}
\label{sec:dpp}
\newcite{Kulesza2012} 
introduced applications and algorithms for
using
{\deter} point processes for machine learning.
Following is a summary of parts of their tutorial
relevant to this paper.

A point process is a probability measure $\mathcal{P}$
on $2^{\mathcal{Y}}$, the set of all subsets of ${\mathcal{Y}}$. 
This point process is {\deter} if the probability measure
satisfies the following property:
if ${\bf Y}$ is a random subset drawn according
to $\mathcal{P}$, then 
 for every subset $A \subseteq {\mathcal{Y}}$,
$$\mathcal{P}(A\subseteq {\bf Y}) = \det(K_A)$$
for some real $N \times N$ matrix $K$, indexed by the elements of ${\mathcal{Y}}$.
 $K_A \equiv [K_{ij}]_{i, j \in A}$ denotes the
restriction of $K$ to the entries indexed by elements of $A$,
  $\det(K_A)$ stands for the determinant 
of matrix $K_A$, and $\det{(K_{\phi}) = 1}$.
Say, $A$ is a set of two elements, \{$i, j$\}.
Using the above formula, 
$$\mathcal{P}(i, j \in {\bf Y}) = K_{ii}K_{jj} - K_{ij}K_{ji} = K_{ii}K_{jj} - K_{ij}^2$$

If the two elements, $i, j$ are similar,
then $K_{ij}$ is large, and the probability 
distribution over the two element set is small.
Therefore, DPPs, by definition, put a greater probability mass
on sets that have \emph{dissimilar} elements,
as compared to sets that have similar elements.
\newcite{Kulesza2012} present a sampling algorithm 
(Algorithm 1, page 16 of their tutorial) for sampling
from a DPP. Given a set of points, this algorithm selects
a subset of the most dissimilar points from the set.
 In spirit, the notion of diversity in the context
 of DPPs  
 is similar to the notion of euclidean distance 
 in the context of {\kmeanspp} (DPP samples diverse points while {\kmeanspp} samples points that are far in terms
 of euclidean distance).
We explore this conceptual connection
between the two sampling techniques
and provide empirical evidence that 
{\kmeansdpp} has advantages over {\kmeanspp}.

The most appealing aspect of the
DPP sampling
 algorithm is that
it is not required that the
number of dissimilar points be known in advance.
Given a set of points, the DPP sampler returns
a subset of dissimilar points. We use the \underline{cardinality
of this sampled subset as $k$} for running the {\kmeans}
clustering algorithm.
The DPP sampling algorithm, in addition, has a version,
called $k$-DPP \cite{Kulesza2012}, in which
one may specify $k$ as the cardinality of the subset of
dissimilar points to be sampled.
When we sample the initial centroids 
for the {\kmeans} clustering algorithm
using $k$-DPP, we refer to the overall scheme
as {\bf {\kkmeansdpp}}.

\begin{figure*}[t!]
\center
a)\includegraphics[width = 2in]{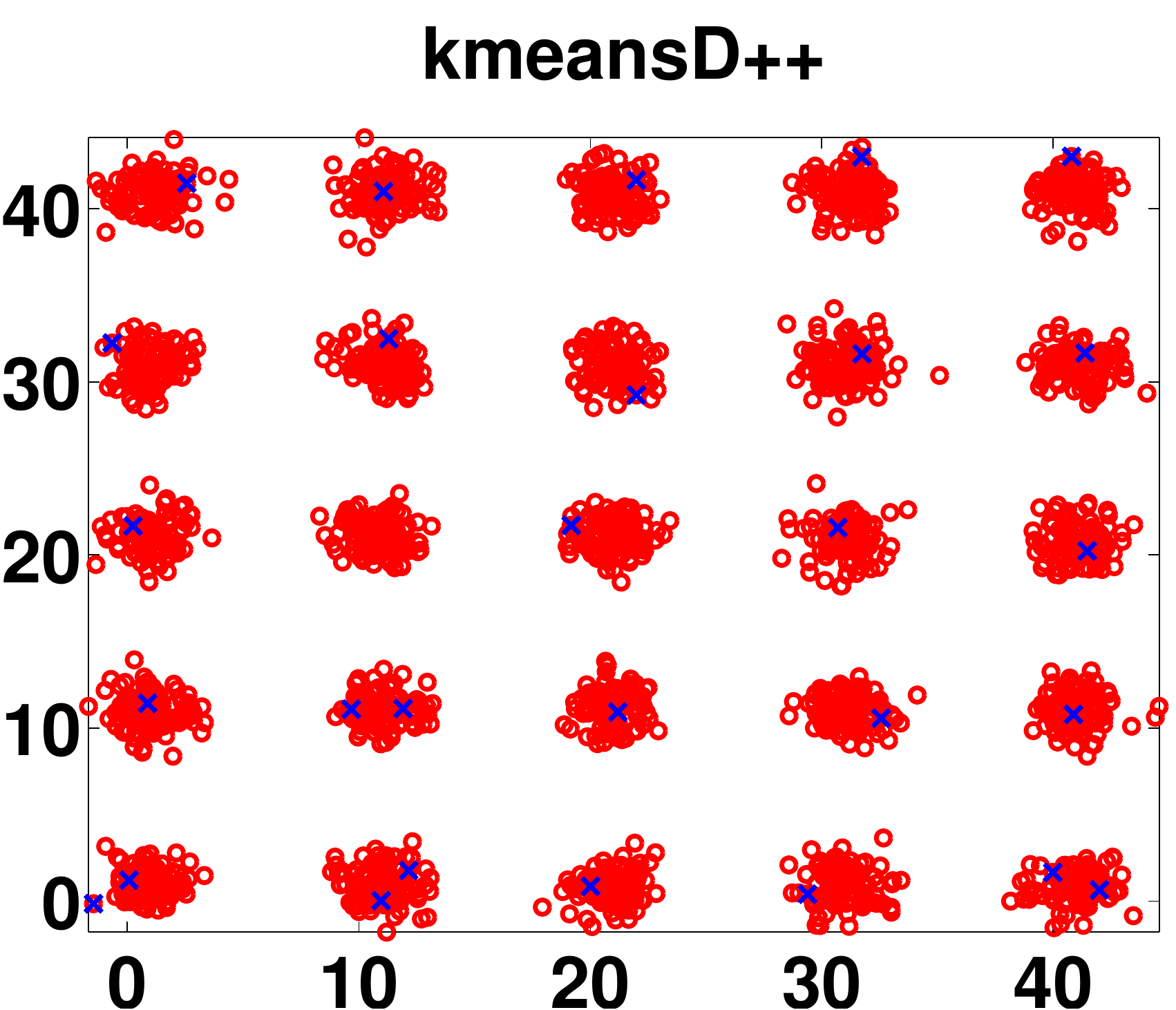}
b)\includegraphics[width = 2in]{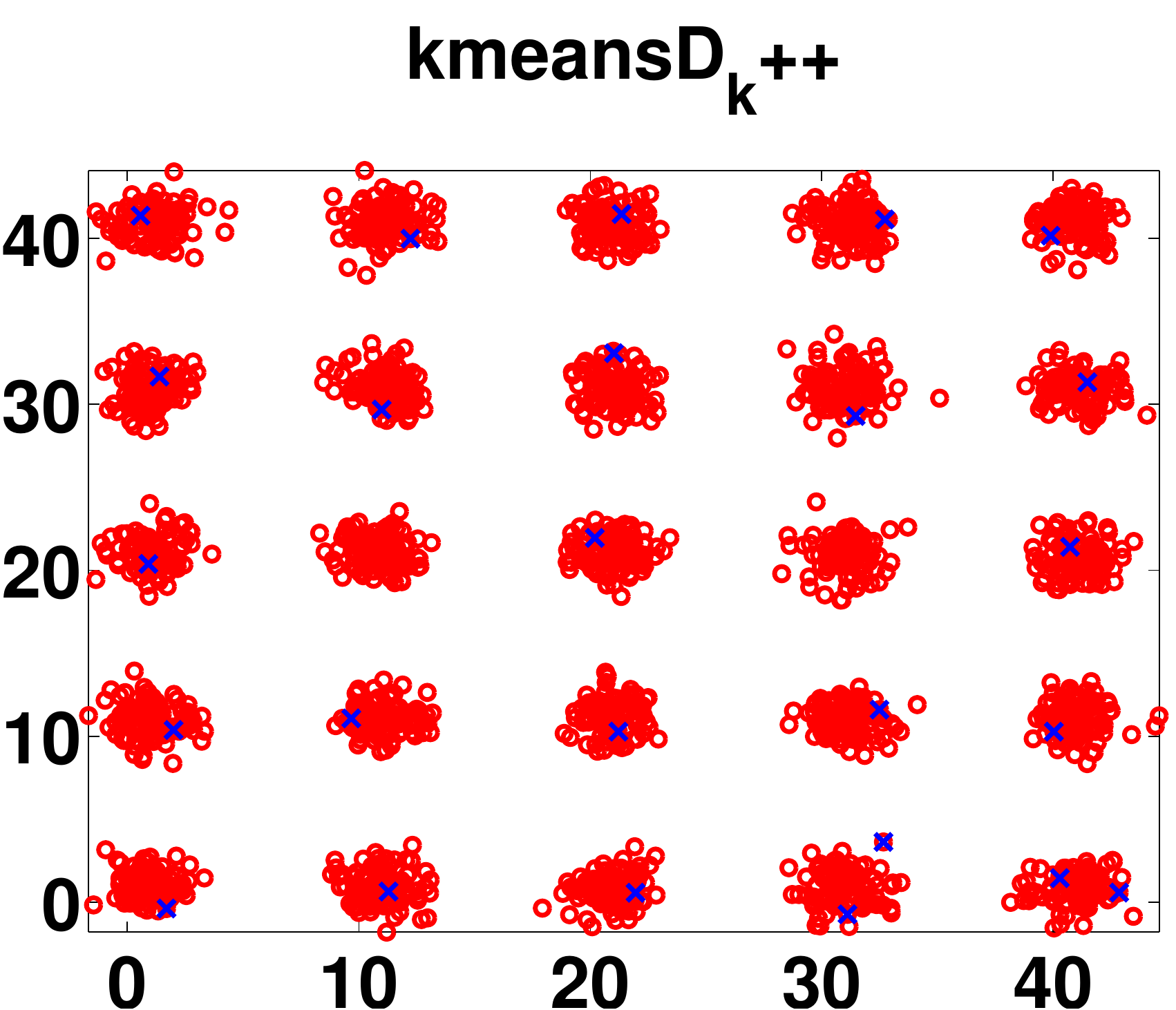}
c)\includegraphics[width = 2in]{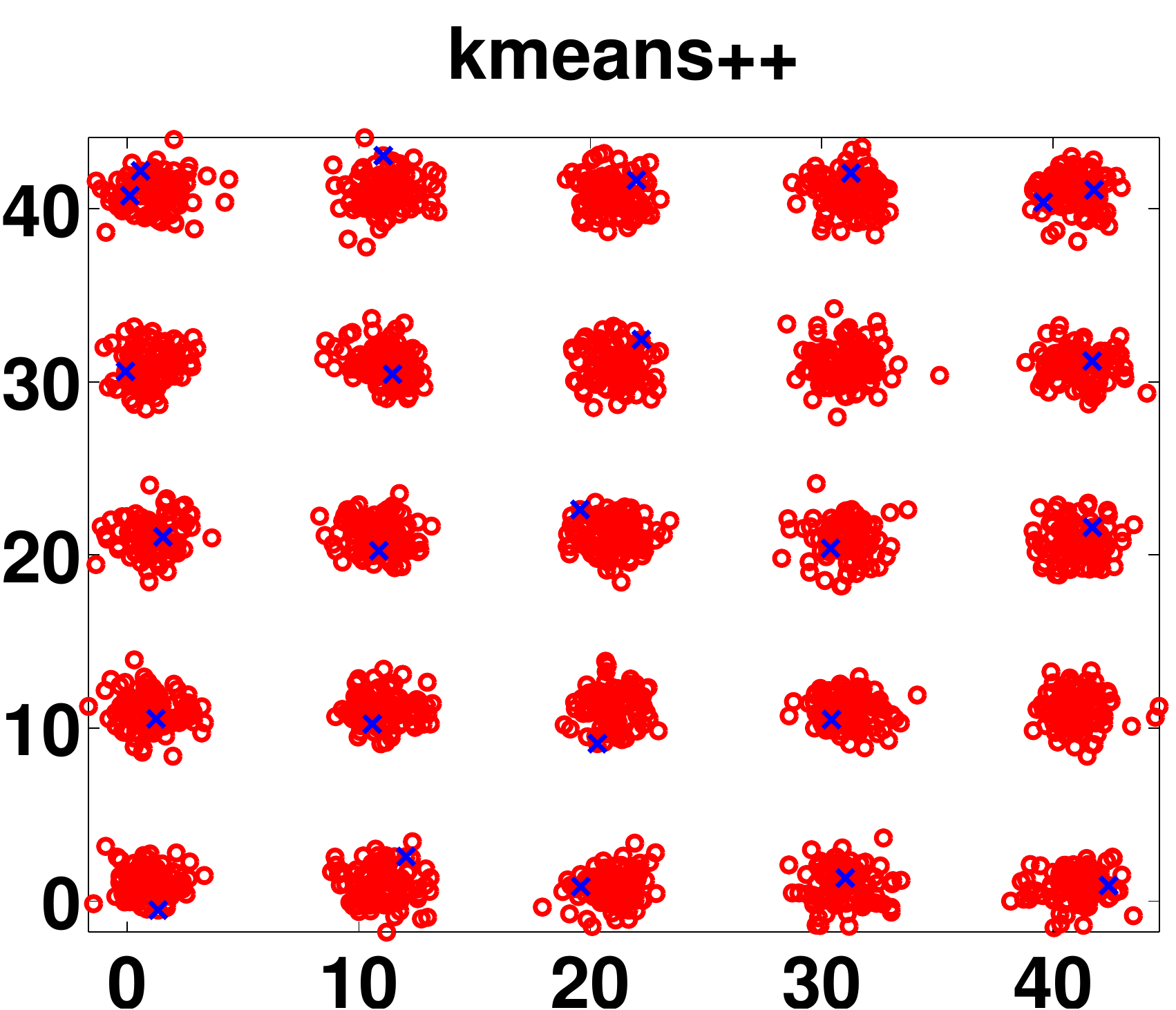}
\caption{The seeds recovered by a) {\kmeansdpp} b) {\kkmeansdpp} and c) {\kmeanspp} initializers on a mixture of $25$ Gaussians.}
\label{fig:synthetic25}
\vspace{-0.2in}
\end{figure*}

\section{{\kmeanspp} versus {\kmeansdpp} }
\label{sec:ppvsDPP}

In this section we will show the fundamental difference between {\kmeanspp} and {\kmeansdpp} initializers. {\kmeanspp} algorithm sampling the initial centroids (also called seeds) for the {\kmeans} algorithm is summarized in Algorithm~$1$. Here, $D(x)$ denotes the shortest Euclidean distance from a data point $x$ to the closest seed from among seeds already chosen ($S$). 
The {\kmeanspp} initializer assigns the highest probability to the data point that is currently the furthest from its closest seed from among the set of seeds chosen already. {\kmeansdpp} initializer chooses the seeds from among the points in the dataset using different probabilities of selecting a new member for set $S$. Before showing the algorithm, we will introduce notation. Let $K$ be the RBF kernel matrix with $(i,j)^{th}$ entrance equal to $K(i,j) = \exp(-\sigma \|x_i - x_j\|^2)$ and $\sigma$ be a fixed positive constant (note that $K$ is of size $n\times n$ and is symmetric positive semi-definite), $K^S$ is a sub-matrix of matrix $K$ of size $|S|\times |S|$ defined by points from $S$, and $K^{S\cup \{x\}}$ is a sub-matrix of matrix $K$ of size $|S|+1\times |S|+1$ defined by points from $S \cup \{x\}$ (both sub-matrices are as well symmetric positive semi-definite). {\kmeansdpp} algorithm is summarized in Algorithm~$2$.\footnote{Note that the practical implementation of the {\kmeansdpp} algorithm differs from the Algorithm~$2$ and follows Algorithm 1 (page 16) from~\cite{Kulesza2012}, however from the perspective of the theoretical analysis the simpler version summarized in Algorithm~$2$ is more convenient.}
\vspace{-0.08in}
\begin{algorithm}
\label{alg:kmeanspp}       
\setlength\tabcolsep{1pt}
\begin{tabular}{ll}
Input: dataset $\mathcal{X}$\\
1) $S = \emptyset$\\
2) Pick a point uniformly at random from $\mathcal{X}$ and add it to $S$.\\
3) for $i = 1:1:k-1$:\\
\:\:\:\:\:\:\:\: a) choose data point $x \in \mathcal{X}$ at random with\\
\:\:\:\:\:\:\:\:\:\:\:\:\: probability $P(x|S) = \frac{D(x)^2}{\sum_{x^{'} \in X}D(x^{'})^2}$\\
\:\:\:\:\:\:\:\: b) $S = S\cup\{x\}$
\end{tabular}
\caption{\kmeanspp}
\end{algorithm}
\vspace{-0.33in}
\begin{algorithm}
\label{alg:kmeansdpp}       
\setlength\tabcolsep{1pt}
\begin{tabular}{ll}
Input: dataset $\mathcal{X}$\\
1, 2 and 4) as in {\kmeanspp}\\ 
3) for $i = 1:1:k-1$:\\
\:\:\:\:\:\:\:\: choose data point $x \in \mathcal{X}$ at random with\\
\:\:\:\:\:\:\:\: probability $P(x|S) = \frac{det(K^{S\cup \{x\}})}{det(K^S)}$
\end{tabular}
\caption{\kmeansdpp}
\end{algorithm}
{\kmeansdpp} favors diversity by putting higher probability to sets of items that are diverse, which is the property that the {\kmeanspp} initializer also has, however the former uses less aggressive initialization scheme, i.e. it does not necessarily put the highest probability to the data point that is currently the furthest from its closest seed from among the set of seeds chosen already. This can be shown by considering a simple example. Let $\mathcal{X} = \{x_1,x_2,x_3\}$ be the set of points on a $1D$ line, where $x_1$ was sampled first and then $x_2$ and $x_3$. We will consider two possible locations for $x_3$, that we will refer to as $x_3^{'}$ and $x_3^{''}$, shown below:\\
\\
a)  $x_1$--------------$0$--------------$x_2$--------------$x_3^{'}$\\
b)  $x_1$--------------$0$--------------$x_2$   and $x_3^{''} = 0$\\
\\
Let $\|x_1\| = \|x_2\| = D$ and $\|x_3^{'} - x_2\| = D-\epsilon$ and let $D$ be fixed such that $D > \sqrt{\frac{\log 6}{4\sigma}}$. 
One can show that the DPP $k$-means initializer will put higher probability to select $x_3^{'}$ then $x_3^{''}$, which is captured in Lemma~\ref{lem:counterexample}. The proof is deferred to the appendix.
\begin{lemma}
There exists $\epsilon \in (0,D)$ such that $P(x_3^{'}|S) > P(x_3^{''}|S)$, thus {\kmeansdpp} initializer can put the highest probability to the point which is not the furthest from the closest seed from among seeds already chosen.
\label{lem:counterexample}
\end{lemma}

\section{Evaluation on synthetic datasets}

\vspace{-0.05in}
\label{sec:eval1}
\begin{table}[ht]
\centering
\setlength\tabcolsep{4pt}
\begin{tabular}{|l|c|c|c|c|c|c||c|}
\hline
$k_t$ & 4 & 9 & 16 & 25 & 36 & 100 & Total\\
\hline
\hline
{\kmeansrand} & 1 & 3 & 6 & 8 & 14 & 31 & 63\\
\hline
{\kmeanspp} & 0 & 1 & 2 & 2 & 2 & 9 & 16\\
\hline
{\kkmeansdpp} & 0 & 1 & 1 & 2 & 4 & 10 & 18\\
\hline
\hline
{\kmeansdpp} & 0 & 0 & 0 & 1 & 0 & 9 & 10\\
\hline
k & 4 & 11 & 18 & 28 & 39 & 105 \\
\cline{1-7}
\end{tabular}
\caption{Comparison of {\kmeansrand}, {\kmeanspp}, {\kkmeansdpp} and {\kmeanspp} initlializers on synthetic datasets. Number of clusters \emph{missed} by each
of the three algorithms. $k_t$ denotes the true number
of clusters. 
}
\label{tab:synthetic}
\vspace{-0.1in}
\end{table}

We  compare {\kmeansrand}, {\kmeanspp}, {\kmeansdpp} and {\kkmeansdpp} initializers on standard synthetic data-sets.
{\kkmeansdpp} refers to the {\kmeansdpp} initializer run with pre-specified number of clusters ($k$).
For these datasets we know the true number of clusters, denoted as $k_{t}$, and we well understand the geometry of the problem. We use mixture of well-separated Gaussians on a 2D grid. The variance of each Gaussian is $1$, the number of points in each of them is $100$ and the separation between them is $10$. The results are presented in Table~\ref{tab:synthetic} (for each experiment we report the median result over $50$ runs). For all the methods we report the number of missing clusters (\textit{missed}). Furthermore, for the {\kmeansdpp} initializer we report the number of clusters recovered automatically ($k)$. Additionally, in Figure~\ref{fig:synthetic25} we show an exemplary result we obtained for a mixture of $25$ Gaussians. The results indicate that the performance of {\kmeansdpp} and {\kmeanspp} initializers are similar and furthermore {\kmeansdpp} initializer is able to recover the true number of clusters underlying the data very accurately without having the number of clusters pre-specified  (the correlation between $k_t$, the true $k$ and the $k$ predicted by {\kmeansdpp} is 0.99). 
This highlights the ability of {\kmeansdpp}
to approximate the true $k$ -- 
an ability that the {\kmeanspp} initializer does not have.

\vspace{-0.05in}
\section{Evaluation on real datasets}
\label{eval2}

In this section, we compare the performance of the {\kmeans} clustering algorithm initialized in two different ways, using the {\kmeansdpp} initializer and using the {\kmeanspp}. The comparison is
presented on three benchmark datasets: \textit{iris, ecoli} and \textit{dermatology}.\footnote{Downloaded from \textsf{archive.ics.uci.edu/ml/datasets.html}.} 
The results are averaged over $50$ runs. 
Table~\ref{table:mnist} 
presents the F1-measures for clustering
the three data-sets using {\kmeanspp} and
{\kkmeansdpp} ({\kmeansdpp} with the number
of clusters pre-specified).
The results show 
that the F1-measures (considering the standard
deviation) for the two clustering
algorithms are comparable,
which implies that {\kmeansdpp} 
is empirically similar to {\kmeanspp}.

\vspace{-0.05in}
\begin{table}[h]
\center
\setlength\tabcolsep{4.5pt}
\begin{tabular}{|c|c||c|c|}
\hline
Datasets & $k_{t}$ & \kmeanspp & {\kkmeansdpp}\\
\hline
\hline
iris & 3 & 0.88\small{$\pm$0.08} & 0.87\small{$\pm$0.10}\\
\hline
ecoli & 8 & 0.56\small{$\pm$0.06} & 0.63\small{$\pm$0.06}\\
\hline
dermatology & 6 & 0.72\small{$\pm$0.12} & 0.68\small{$\pm$0.14}\\
\hline
\end{tabular}
\vspace{-0.05in}
\caption{F1-measure obtained by {\kmeanspp} and {\kkmeansdpp} on benchmark datasets.}
\label{table:mnist}
\vspace{-0.1in}
\end{table}

To highlight that
{\kmeansdpp} is able to 
automatically approximate the true $k$
while maintaining a good clustering
performance,
 we compare the value of the
{\kmeans} clustering objective (called \emph{cost}, lower is better)
of {\kmeansdpp} and {\kkmeansdpp}.
Note, we cannot report F1-measures for this
evaluation since {\kmeansdpp} automatically
selects the number of clusters, which can be
different from the true number of clusters. 

\vspace{-0.05in}
\begin{table}[h]
\center
\setlength\tabcolsep{0.5pt}
\begin{tabular}{|c|c||c|c|c|}
\hline
Data & $k_{t}$ & \multicolumn{2}{c|}{\kmeansdpp} & {\kkmeansdpp}\\
\cline{3-5}
 & & k & cost & cost\\
\hline
\hline
iris & 3 & 3.80\small{$\pm$0.41} & 62.60\small{$\pm$9.19} & 92.94\small{$\pm$27.16}\\
\hline
ecoli & 8 & 6.23\small{$\pm$0.89} & 22.42\small{$\pm$2.75} & 18.64\small{$\pm$1.68}\\
\hline
derm & 6 & 32.63\small{$\pm$0.52} & 2122.31\small{$\pm$27.36} & 3824.52\small{$\pm$282.51}\\
\hline
\end{tabular}
\vspace{-0.05in}
\caption{Performance of {\kmeansdpp}
and {\kkmeansdpp} on benchmark datasets.
$k_t$ is the true number of clusters.
$k$ is the number of clusters automatically 
approximated by {\kmeansdpp}.
\emph{cost} is the value of the {\kmeans}
objective.}
\label{tab:autovsnonauto}
\vspace{-0.1in}
\end{table}
 
Table~\ref{tab:autovsnonauto} shows two results:
1) the $k$ predicted by {\kmeansdpp}
is close the true $k_t$ (columns 2 and 3)
and 
2) the quality of clustering in terms of the
cost of {\kmeansdpp} and {\kkmeansdpp}
is comparable. The exception is the \textit{dermatology} (\textit{derm}) dataset, for which interestingly every feature has $34$ attributes which is very close to the number of clusters that {\kmeansdpp} recovered. Since the DPP sampling
 algorithm uses the eigen-value decomposition,
 it seems that the sampler is mis-lead in
 thinking that data-set has $\sim34$ classes. 
This behavior of the DPP sampler is interesting
and requires further investigation (perhaps it is caused by weakly dependent features). Simultaneously, the $k$-means cost of the clusterings recovered by {\kmeansdpp} on the \textit{dermatology} dataset is significantly lower than the cost of {\kkmeansdpp}. Note that it can be justified by the fact that when {\kmeansdpp} resp. largely overestimates/underestimates $k$, the $k$-means cost of {\kmeansdpp} should be resp. lower/higher than {\kkmeansdpp} because choosing resp. larger/smaller $k$ typically implies resp. smaller/larger average distance of a data point to its closest cluster center.

\vspace{-0.07in}
\section{Evaluation on a Real Text Clustering Task}
\label{sec:eval3}

In Anonymous 2014, we introduced a novel
task 
of automatically drawing \emph{xkcd}
movie narrative charts (right half of Figure 2) 
from textual screenplays (top left of Figure 2). 
We presented an end-to-end pipeline,
employing algorithms from natural language
processing, social network analysis and 
machine learning literature. 
The main focus
of Anonymous 2014 was to present a novel
task, its motivation, and  a basic system
pipeline. However, in this paper,
 we are only concerned with improving
the key component of the pipeline
-- the text clustering module. 

While
for other text clustering tasks, heuristically setting
$k$ may not
be a major limitation, for the task at hand, it is
critical that we have an automatic way of selecting 
(or approximating) $k$.
This is because, in trying to cluster one data-set,
it is well justified to use domain knowledge
and human intuition to set $k$
or to refine
$k$ by observing the output.
However, for the task at hand, we need to find a clustering
per movie. Since there are hundreds of movies,
each with unique characteristics, 
heuristically setting $k$ is not feasible.


\begin{figure*}[ht!]
\center
\includegraphics[width=3.2in]{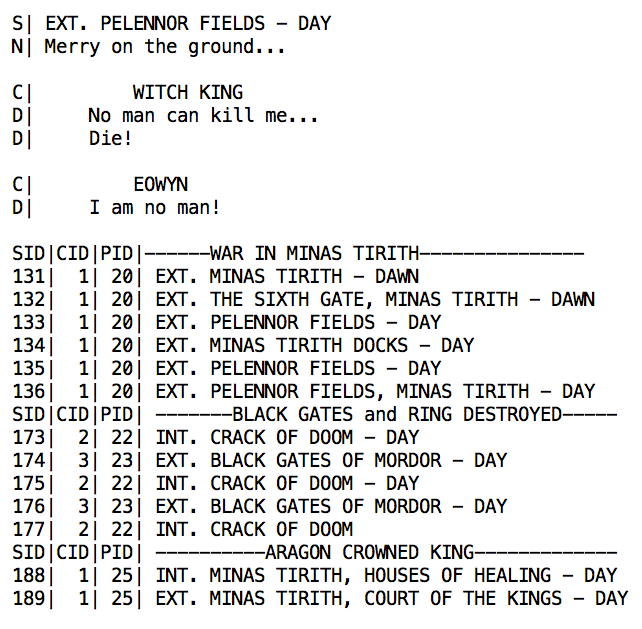}
\includegraphics[width=3.2in]{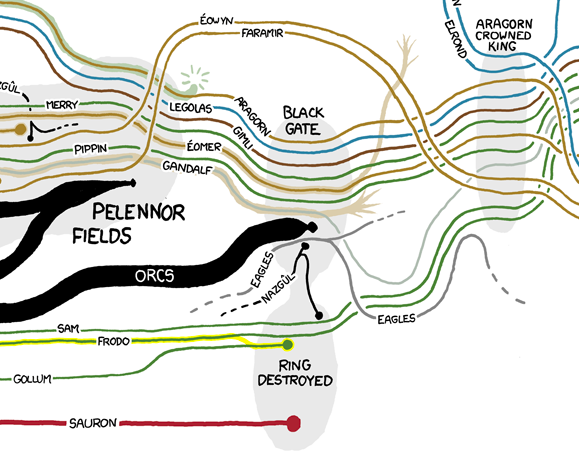}
\vspace{-0.1in}
\caption{Right half: \emph{xkcd} movie narrative chart for part of the movie \emph{Lord of the Rings}.
These charts show character interactions. The horizontal axis is time. The vertical grouping of lines indicates which characters are together at a given time.
Source of image: \url{http://xkcd.com/657/large/}
Left top: snippet from the textual screenplay used as input
for automatically creating the chart.
Left bottom: scene boundaries and their scene identifier (SID),
their cluster identifier (CID), and their plot identifier (PID).}
\label{fig:xkcd}
\vspace{-0.1in}
\end{figure*}
\begin{table*}[ht!]
\setlength{\tabcolsep}{9pt}
\centering
\begin{tabular}{|l|c|c|c|c|c|c|}
\hline
Movie & \# scenes (n) & \# locations (gold $k$) & log(n) & $\sqrt{n}$ & $k_{\kmeansdpp}$\\
\hline
Star Wars & 137 & 35  & 2.13 & 11.7 & 41.98\small{$\pm$3.30} \\
\hline
The Last Crusade & 148 &  57  & 2.17 & 12.16 & 47.72\small{$\pm$3.50} \\
\hline
 Raiders of the Lost Ark & 139 &  73 & 2.14 & 11.78 & 51.56\small{$\pm$5.05}\\
\hline
Pirates of the Caribbean & 140 &  23 & 2.14 & 11.83 & 41.24\small{$\pm$4.32}\\
\hline
The Bourne Identity & 160 & 74  & 2.20 & 12.64 & 61.98\small{$\pm$5.03}\\
\hline
Batman & 209 & 77  & 2.32 & 14.45 & 71.42\small{$\pm$5.14}\\
\hline
\hline
Correlation with gold $k$ & 0.58 & 1 & 0.59 & 0.58 & 0.84\\
\hline
\end{tabular}
\vspace{-0.05in}
\caption{List of movies, the number of scene boundaries,
and the number of unique locations per movie in our  test set (first three columns). Automatically selected $k$ by commonly used heuristic functions (next two columns). Automatically selected $k$ by using DPPs: mean and standard deviation over 50 runs (last two columns). Last row of the table shows the correlation between predicted $k$ and gold $k$.}
\label{table:data}
\vspace{-0.2in}
\end{table*}

\subsection{Terminology and Task Definition}
\label{subsec:terminology}

\newcite{Turetsky2004} describe the structure of a movie
screenplay. A screenplay is written using a strict formatting grammar.
It has \emph{scene boundaries} that textually separate
 scenes of a movie. 
 Figure~\ref{fig:xkcd}
 shows some of the scene boundaries from the movie
 \emph{The Lord of the Rings}.  
A scene boundary indicates whether the scene is
to take {place inside or outside} (INT, EXT),
the {name of the location},
and can potentially specify the {time of day} (e.g. DAY or NIGHT).
The clustering task is to cluster scene boundaries
(based on their lexical similarity) into $k$ clusters (with $k$ unknown).
Since scene boundaries specify the location at which
a scene is shot, the goal is to automatically determine
the number and description of different scene locations 
in a movie (we remove
tags INT./EXT., DAY/NIGHT before clustering). 

Scene locations mentioned in scene boundaries are lexically similar, but not exactly the same. This is because a scene boundary, more often than not, describes a scene location, along with sub-location(s). For example, in Figure~\ref{fig:xkcd}, the scene location Minas Tirith, which is a city, has multiple sub-locations such as ``DOCKS'' and  ``HOUSES OF HEALING''. Moreover, there are inconsistencies in the scene location descriptions. For example, some scene location descriptions for Pelennor Fields, which is a sub-location associated with Minas Tirith, are present as ``PELENNOR FIELDS/MINAS TIRITH'', whereas others are present as ``PELENNOR FIELDS''. As a consequence, a simple exact string matching algorithm is insufficient to find scene boundaries that belong to one location.


\vspace{-0.1in}
\subsection{Data}
\label{subsec:data}

To prepare a gold standard for this evaluation, 
we trained two human annotators to read a screenplay 
and mark all scenes (or scene boundaries) 
that belong to one location with a unique integer (which
we refer to as cluster identifier).
For example, in Figure~\ref{fig:xkcd},
one of our annotators marked scene boundaries (\textsc{sid})
from 131 through 136 with cluster (or location)
 identifier (\textsc{cid}) 1.
 This means that all these scenes take place at one location,
 namely \textsc{Minas Tirith}. While performing the annotation
 task, the annotators used world knowledge that \textsc{Pelennor Fields}
 is a sub-location of \textsc{Minas Tirith} and thus should
 be marked with the same cluster identifier.  
 Since we are clustering based on lexical similarity, 
 to put lexically dissimilar strings \textsc{Pelennor Fields}
 and \textsc{Minas Tirith} together, 
 our algorithm relies on the fact that they 
  are mentioned together in
 a few scene boundaries (as is the
 case -- see scene number 136). 
 
After a few rounds of training
we asked our annotators 
to fully annotate the screenplay for the movie \emph{Pirates of the Caribbean: Dead Man's Chest}.
They achieved a high agreement of 0.86.
We then asked our annotators to divide the
remaining set of screenplays into half,
each responsible for one half.

Table~\ref{table:data} gives the list of movies
we annotated, along with 
the number of scenes and number of locations  
in each movie.
We use these screenplays for evaluating our
methodology.


\subsection{Evaluation and Results}
\label{subsec:eval}

We calculate lexical (or string) similarity using a contiguous
word kernel \cite{Lodhi2002}.
We compare three ways of sampling the initial
centroids for the {\kmeans} algorithm: {\kmeansrand}, 
{\kmeanspp}, and {\kmeansdpp}. 

To set $k$ for {\kmeansrand}, we employ
 common heuristics used in the literature:
  $k = \log(n)$ or $\sqrt{n}$, 
where $n$
is the number of data points. Table~\ref{table:data}
presents the predicted number of $k$
using the functions $\log(n)$, $\sqrt{n}$.
We run {\kmeansdpp}  50 times
and report the mean and standard deviation 
of the number of initial centroids selected by the
DPP sampling algorithm automatically. The last row of 
table~\ref{table:data} shows the correlation of
the predicted $k$ with the gold $k$ for the
three methods.\footnote{Multiplying or adding a constant to the functions $\log(n)$, $\sqrt{n}$
will not change the correlation.}
Deciding $k$ using DPPs has a significantly 
higher correlation with the gold $k$ (0.84)
as compared to other standard methods (0.59 and 0.58).
Note that the correlation of the number of scenes
and the gold $k$ is low (0.58), so any monotonic function
of the number of data-points will not have a much
different correlation. This result shows that DPPs are well-suited for choosing $k$ for this data-set.

Next, we show that even if we provide the
{\kmeans} algorithm with the gold $k$, 
sampling using DPPs provides a better initialization,
which results in a better clustering. 
Table~\ref{table:fmeas} shows
the macro-F1-measures for clustering obtained
by three different ways
of sampling the initial centroids. 
The numbers show that sampling using DPPs
results in a significantly better clustering (higher F1-measure).
\vspace{-0.07in}
\begin{table}[h!]
\setlength{\tabcolsep}{3.2pt}
\centering
\begin{tabular}{|l|c|c|c|c|}
\hline
Movie & k & {\small \kmeansrand} & {\small \kmeanspp} & {\small \kkmeansdpp}\\
\hline
Star Wars & 35 & 0.61\small{$\pm$0.04} & {0.62}\small{$\pm$0.02} & {\bf 0.63}\small{$\pm$0.04}\\
\hline
Crusade & 57 & 0.80\small{$\pm$0.04} & { 0.84}\small{$\pm$0.02} & {\bf 0.86}\small{$\pm${0.02}}\\
\hline
Raiders & 73 & 0.68\small{$\pm$0.03} & { 0.76}\small{$\pm${ 0.02}} & {\bf 0.77}\small{$\pm${ 0.02}}\\
\hline
Pirates & 23 & 0.62\small{$\pm$0.04} & {\bf 0.63}\small{$\pm$0.02} & 0.61\small{$\pm$0.04}\\
\hline
Bourne & 74 & 0.64\small{$\pm${0.03}} & { \bf 0.69}\small{$\pm$0.03} & { 0.68}\small{$\pm$0.05}\\
\hline
Batman & 77 & 0.62\small{$\pm$0.03} & { 0.63}\small{$\pm$0.02} & {\bf 0.66}\small{$\pm$0.03}\\
\hline
\end{tabular}
\vspace{-0.05in}
\caption{Mean and standard deviations of F1-measure on the test set.}
\label{table:fmeas}
\vspace{-0.3in}
\end{table}

\section{Conclusion and Future Work}
We conclude that {\kmeansdpp} compares favorably to 
{\kmeanspp} and performs better than {\kmeansrand} 
with two additional advantages: it may be used in scenarios 
where explicit feature representation is absent and
where the $k$ is unknown.
In the future, we will attempt to prove approximation guarantees with respect to the $k$-means clustering objective for the {\kmeansdpp} algorithm. 

\section{Appendix}
First, we will show a useful lemma that we will use later.

\begin{lemma}
There exists $\epsilon \in (0,D)$ such that
\begin{equation}
\exp(-2\gamma D^2(1\!-\!\frac{\epsilon}{D})^2) - \exp(-2\gamma D^2) \!<\! \frac{\exp(-2\gamma D^2)}{2}
\end{equation}
\label{lem:one}
\end{lemma}
\begin{proof}
For a fixed $D$ this result is straight-forward.
\end{proof}

\begin{proof}[Proof of Lemma~\ref{lem:counterexample}]

$x_1$, $x_2$ and $x_3$ are respectively the first, second and third data point chosen by the {\kmeansdpp} initializer. Thus we have that
\begin{equation*}
P(x_2|S = x_1) = \frac{det(K^{x_2\cup x_1})}{det(K^{x_1})}= 1 - \exp(-2\sigma \|x_1 - x_2\|^2)
\end{equation*}
and
\[P(x_3|S = x_1\cup x_2) = \frac{det(K^{x_1\cup x_2\cup x_3})}{det(K^{x_1\cup x_2})}
\]
\[= 1 - \frac{1}{1 - \exp(-2\sigma \|x_1 - x_2\|^2)}\cdot \{\exp(-2\sigma\|x_2 - x_3\|^2)
\]
\[+\exp(-2\sigma\|x_1 - x_3\|^2)
\]
\[-2\exp(-\sigma(\|x_1 - x_2\|^2+\|x_2 - x_3\|^2 + \|x_1 - x_3\|^2))\}
\]
Note that since $\|x_1 - x_2\| = 2D$, $\|x_3^{'} - x_2\| = D-\epsilon$, $\|x_3^{'} - x_1\| = 3D-\epsilon$, $\|x_3^{''} - x_2\| = D$ and $\|x_3^{''} - x_1\| = D$, the following chain of inequlities are equivalent:
\[P(x_3^{'}|S) > P(x_3^{''}|S)
\]
\[\Longleftrightarrow \exp(-2\sigma(3D-\epsilon)^2) + \exp(-2\sigma(D-\epsilon)^2)
\]
\[ - 2\exp(-\sigma(4D^2+(3D-\epsilon)^2 + (D-\epsilon)^2))
\]
\[\leq 2\exp(-2\sigma D^2) - 2\exp(-6\sigma D^2)
\]
\[\Longleftrightarrow \exp(-18\sigma D^2(1-\frac{\epsilon}{3D})^2) + \exp(-2\sigma D^2(1-\frac{\epsilon}{D})^2)
\]
\[ - 2\exp(-\sigma(4D^2+(3D-\epsilon)^2 + (D-\epsilon)^2))
\]
\[\leq 2\exp(-2\sigma D^2) - 2\exp(-6\sigma D^2) 
\]
We want to prove that the last inequality holds. We will show that by instead showing the series of stronger inequalities that hold and imply the above one. Note, that the inequality that implies the above one is given below
\begin{eqnarray}
\exp(-18\sigma D^2(1-\frac{\epsilon}{3D})^2) + \exp(-2\sigma D^2(1-\frac{\epsilon}{D})^2) \nonumber\\
 \leq 2\exp(-2\sigma D^2) - 2\exp(-6\sigma D^2) 
\label{eq:firstineq}
\end{eqnarray}
This inequality can be rewritten as
\[2\exp(-6\sigma D^2) + \exp(-18\sigma D^2(1-\frac{\epsilon}{3D})^2)
\]
\[ + \exp(-2\sigma D^2(1-\frac{\epsilon}{D})^2) \leq 2\exp(-2\sigma D^2) 
\]
Recall that $\epsilon < D$ thus
\[1 - \frac{\epsilon}{3D} > \frac{2}{3} \Longleftrightarrow \exp(-18\sigma D^2(1 - \frac{\epsilon}{3D})^2) < \exp(-8\sigma D^2)
\]
Thus, an even stronger inequality than the one in Equation~\ref{eq:firstineq} is the following one
\begin{eqnarray}
2\exp(-6\sigma D^2) + \exp(-8\sigma D^2) \nonumber\\
 + \exp(-2\sigma D^2(1-\frac{\epsilon}{D})^2) < 2\exp(-2\sigma D^2) 
\label{eq:secondineq}
\end{eqnarray}
The inequality in Equation~\ref{eq:secondineq} implies the inequality in Equation~\ref{eq:firstineq}. Note that $\exp(-8\sigma D^2) \leq \exp(-6\sigma D^2)$ thus one can construct an even stronger inequality given in Equation~\ref{eq:thirdineq}, than the one in Equation~\ref{eq:secondineq} that directly implies Equation~\ref{eq:secondineq} and therefore also Equation~\ref{eq:firstineq}. 
\begin{equation}
3\exp(\!-6\sigma D^2) \!+\! \exp(\!-2\sigma D^2(1\!\!-\!\!\frac{\epsilon}{D})^2) \!\leq\! 2\exp(\!-2\sigma D^2) 
\label{eq:thirdineq}
\end{equation}
Recall that 
\begin{equation*}
D > \sqrt{\frac{\log 6}{4\sigma}} \Longleftrightarrow 3\exp(-6\sigma D^2) < \frac{1}{2}\exp(-2\sigma D^2)
\end{equation*}
Finally, we will below provide the last inequality, in Equation~\ref{eq:fourthineq}, which is the strongest from all discussed before as, if it holds, it directly implies the inequalities in Equation~\ref{eq:thirdineq} and therefore also  Equation~\ref{eq:secondineq} and~\ref{eq:firstineq}.
\begin{equation}
\frac{\exp(-2\sigma D^2)}{2} + \exp(-2\sigma D^2(1-\frac{\epsilon}{D})^2) < 2\exp(-2\sigma D^2) 
\label{eq:fourthineq}
\end{equation}
This equality can be equivalently rewritten as
\begin{equation*}
\exp(-2\sigma D^2(1-\frac{\epsilon}{D})^2) - \exp(-2\sigma D^2) < \frac{\exp(-2\sigma D^2)}{2},
\end{equation*}
where the last inequality holds by Lemma~\ref{lem:one}.
\end{proof}

\newpage
\bibliographystyle{aaai}
\bibliography{PAPER}


\end{document}